\documentclass{article}
\usepackage{kr}

\usepackage{times}
\usepackage{soul}
\usepackage{url}
\usepackage[hidelinks]{hyperref}
\usepackage[utf8]{inputenc}
\usepackage[small]{caption}
\usepackage{graphicx}
\usepackage{stmaryrd} %
\usepackage{amsmath}
\usepackage{amsthm}
\usepackage{booktabs}
\usepackage{algorithm}
\usepackage{algorithmic}
\usepackage{amssymb}

\newcommand{\sat}{\models}

\newcommand{\nsat}{\hspace{2px}\mid\hspace{-2px}\neq\hspace{1px}}

\newcommand{\cond}[2][]{(\ifthenelse{\equal{#1}{}}{#2}{#1}|#2)}

\renewcommand{\phi}{\varphi}

\newcommand{\tthen}{\text{ then }}
\newcommand{\tand}{\text{ and }}

\newcommand{\tall}{\text{ for all }}
\newcommand{\twith}{\text{ with }}

\newcommand{\krs}{\kappa_{|\Sigma'}}

\newcommand{\dfp}[2][]{\textbf{(DFP-#2)#1}}

\newcommand{\dfpessr}[1]{\dfpes{#1}{$_\Sigma$}}
\newcommand{\dfpesl}[1]{\dfpes{#1}{$_\mathcal{L}$}}

\newcommand{\dfpes}[1]{\textbf{(DFPes-#1)}}

\newcommand{\forgetL}{\circ_f^\mathcal{L}}
\newcommand{\forgetS}{\circ_f^\Sigma}
\newcommand{\forgetSm}{\circ_f^{\Sigma , m}}

\newcommand{\lb}{\llbracket}
\newcommand{\rb}{\rrbracket}

\newcommand{\pt}{\hspace{1pt}}

\usepackage{multirow}
\usepackage{xcolor}

\usepackage{cleveref} %
\crefname{definition}{Def.}{Def.}
\crefname{theorem}{Th.}{Th.}
\crefname{corollary}{Cor.}{Cor.}
\crefname{proposition}{Prop.}{Prop.}
\crefname{lemma}{Lem.}{Lem.}
\crefname{example}{Ex.}{Ex.}
\crefname{table}{Tab.}{Tab.}
\crefname{observation}{Obs.}{Obs.}

\newtheorem{definition}{Definition}
\newtheorem{theorem}{Theorem}
\newtheorem{corollary}{Corollary}
\newtheorem{proposition}{Proposition}
\newtheorem{lemma}{Lemma}
\newtheorem{example}{Example}

\title{Forgetting Formulas and Signature Elements in Epistemic States}

\author{%
Alexander Becker$^1$\and
Gabriele Kern-Isberner$^1$\and
Kai Sauerwald$^2$\and
Christoph Beierle$^2$
\affiliations
$^1$TU Dortmund University, Germany\\
$^2$FernUniversit\"at in Hagen, Germany
\emails
\{alexander2.becker, gabriele.kern-isberner\}@tu-dortmund.de,\\
\{christoph.beierle, kai.sauerwald\}@fernuni-hagen.de
}

\begin{document}

\maketitle

\begin{abstract}
	Delgrande's knowledge level account of forgetting provides a
	general approach to forgetting syntax elements from sets of formulas
	with links to many other forgetting operations, in particular, to
	Boole's variable elimination. On the other hand, marginalisation of
	epistemic states is a specific approach to actively reduce signatures in
	more complex semantic frameworks, also aiming at forgetting atoms that
	is very well known from probability theory. In this paper, we bring
	these two perspectives of forgetting together by showing that
	marginalisation can be considered as an extension of Delgrande's
	approach to the level of epistemic states. More precisely, we generalize
	Delgrande's axioms of forgetting to forgetting in epistemic states, and
	show that marginalisation is the most specific and informative
	forgetting operator that satisfies these axioms. Moreover, we elaborate
	suitable phrasings of Delgrande's concept of forgetting for formulas by
	transferring the basic ideas of the axioms to forgetting formulas from
	epistemic states. However, here we show that this results in trivial
	approaches to forgetting formulas. This finding supports the claim that
	forgetting syntax elements is essentially different from belief
	contraction, as e.g. axiomatized in the AGM belief change framework.
\end{abstract}

\section{Introduction}
In the past decade, the popularity and presence of artificial intelligence (AI) grew rapidly and thereby reached almost every part of our daily lives. From product and media recommendations, voice assistants, and smart homes over industrial optimizations, medical research, and traffic, to even criminal prosecution. And most probably, the importance of AI will grow even further in the near future, due to the ever-increasing amount of data that accumulates day by day and the huge potential it carries. However, so far only little attention was given to the concept of forgetting, even though it plays an essential role in many areas of our daily lives as well. In 2018 the General Data Protection Regulation (GDPR) became applicable, which gives every citizen of the European Union the right to be forgotten (GDPR - Article 17).
This raises the question what it actually means to forget something, and whether it is sufficient to only delete some data in order to forget certain information. This is clearly not the case, since AI systems fitted on this data might still be able to infer the information we like to forget. Thus, forgetting is far more complex than just deleting data. From a cognitive point of view, forgetting is an inextricable part of any learning process that helps handling information overload, sort out irrelevant information, and resolve contradictions. Moreover, it is also of importance when it comes to knowledge management in organisational contexts \cite{kluge2019investigating}, socio-digital systems \cite{ellwart2019intentional}, and domains with highly dynamic information such as supply chain and network management. These few examples illustrate the importance of forgetting in AI systems to guarantee individual privacy and informational self-determination, but also efficient reasoning by blinding out irrelevant information.

In the domain of logic and knowledge representation, several logic-specific forgetting definitions exist, e.g. Boole's variable elimination \cite{boole1854investigation},  fact forgetting in first-order logic \cite{lin1994forget} and forgetting in modal logic \cite{baral2005knowledge}. However, none of these specific approaches argued about the general notions of forgetting, but rather provided a way to compute its result. In \cite{delgrande2017knowledge}, Delgrande presented a general forgetting approach with the goal to unify many of the hitherto existing logic-specific approaches. Moreover, he stated a set of properties he refers to as \emph{right} and \emph{desirable} when it comes to the notions of forgetting. In contrast to Delgrande's approach, Beierle et al. \shortcite{beierle2019towards} presented a general framework for cognitively different kinds of forgetting, which also consider the common-sense understanding of forgetting, and their realisation by means of ordinal conditional functions. In the following, we take this broad, common-sense motivated view of forgetting in contrast to the viewpoint put forward by Delgrande, who explicitly states that e.g. the belief change of contraction should not be considered as forgetting.

In this work, we show that Delgrande's forgetting approach is included in and even generalized by the cognitively different kinds of forgetting presented in \cite{beierle2019towards}, concretely by means of the marginalisation. Moreover, we show that the forgetting properties Delgrande refers to as \emph{right} and \emph{desirable} are not suitable to axiomatise the general properties of all kinds of forgetting, but only of those that aim to forget signature elements instead of formulas. Thus, the here presented results form another step towards a general framework for different kinds of forgetting, and provide a deeper understanding of their properties and inherent differences.

Finally, we want to give an overview of how this work is structured. In \Cref{sec:formal basics}, we give all the preliminaries needed in the later sections including model theoretical basics and ordinal conditional functions. Then we will present both of the above-mentioned general forgetting approaches in \Cref{sec:forgetting} and show that the marginalisation extends Delgrande's forgetting to epistemic states, since both approaches always result in the same posterior beliefs. In \Cref{sec:postulates}, we will then generalize and extend the properties stated by Delgrande to epistemic states, and show that the marginalisation satisfies all of them. Moreover, we show that the marginalisation is the most specific approach satisfying these properties. Finally, we extend the same properties to forgetting formulas in epistemic states and show that they are not suitable for axiomatizing general properties of forgetting, since they imply trivial approaches of forgetting formulas. In \Cref{sec:conclusion}, we present our conclusions as well as some outlooks for future works.

\section{Formal Basics}\label{sec:formal basics}
In the following, we introduce the formal basics as needed in this work. With $\mathcal{L}_\Sigma$ we state a propositional language over the finite signature $\Sigma$ with formulas $\phi, \psi \in \mathcal{L}_\Sigma$. The corresponding interpretations are denoted as $\Omega_\Sigma$. The interpretations $\omega \in \Omega_{\Sigma}$ that satisfy a formula $\phi \in \mathcal{L}_\Sigma$, i.e. $\omega \models \phi$, are called models of $\phi$ and are denoted as $\lb \phi \rb_\Sigma$. If the signature of a model set is unambiguously given by the context, we also write $\lb \phi \rb$ instead. The explicit declaration of the corresponding signature is of particular importance when arguing about different (sub-)signatures. Moreover, each model $\omega \in \Omega_{\Sigma}$ can also be considered as a conjunction of literals corresponding to the truth values $\omega$ assigns to each signature element $\rho \in \Sigma$. Thus, we can also write $\omega \sat \omega'$, where $\omega, \omega' \in \Omega_{\Sigma}$, but $\omega'$ is considered to be the conjunction of literals corresponding to the interpretation. Note that we will make use of this notation several times in this paper. When we specifically want to argue about some signature elements in an interpretation $\omega \in \Omega_\Sigma$, we denote those signature elements $\rho \in \Sigma$ as $\dot{\rho}$ for which the concrete truth assignment is not needed, e.g. $p\dot{b}\dot{f} \in \Omega_{\Sigma}$ with $\Sigma = \{p,b,f\}$.
For two formulas $\phi, \psi \in \mathcal{L}_\Sigma$, we say that $\phi$ infers $\psi$, denoted as $\phi \models_\Sigma \psi$, if and only if $\lb \phi \rb \subseteq \lb \psi \rb$. In case that both model sets are equal, $\phi$ and $\psi$ are equivalent, i.e. $\phi \equiv \psi$, iff $\phi \models \psi$ and $\psi \models \phi$.
Furthermore, the deductively closed set of all formulas that can be inferred from a formula $\phi \in \mathcal{L}_\Sigma$ is given by $Cn_\Sigma(\phi) = \{ \psi \in \mathcal{L}_\Sigma \mid \phi \models_\Sigma \psi \}$. Again, the signature in the index of the $Cn$ operator as well as $\models$ can be omitted when its clearly given by the context. Notice that a formula $\phi \in \mathcal{L}_\Sigma$ is always equivalent to its deductive closure, since their models are equal. The deductive closure $Cn_\Sigma(\phi)$ of a formula $\phi \in \mathcal{L}_\Sigma$ can also be expressed by means of the theory $Th(\lb \phi \rb) = \{ \psi \in \mathcal{L}_\Sigma \mid \lb \phi \rb \sat \psi \}$ of its models $\lb \phi \rb$.
All of the above-mentioned formal basics also hold for sets of formulas $\Gamma \subseteq \mathcal{L}_\Sigma$.

In order to argue about inferences and models in different (sub-)signatures, further basic terms are needed. For two interpretations $\omega, \omega' \in \Omega_\Sigma$, we say that $\omega$ and $\omega'$ are elementary equivalent with the exception of the signature elements $P$, denoted as $\omega \equiv_P \omega'$, if and only if they agree on the truth values they assign to all signature elements in $\Sigma \setminus P$ \cite{delgrande2017knowledge}. Furthermore, we define the reduction and expansion of models in \Cref{def:reduct_extension}, which allow us to argue about models in sub- or super-signatures as well. 

\begin{definition}\label{def:reduct_extension}\textnormal{\cite{delgrande2017knowledge}}
	Let $\Sigma' \subseteq \Sigma$ be signatures and $\phi \in \mathcal{L}_\Sigma$, $\phi' \in \mathcal{L}_{\Sigma'}$ formulas. The \emph{reduction to $\Sigma'$} of models $\lb \phi \rb_\Sigma$ is defined as
	\begin{align*}
	& (\lb \phi \rb_\Sigma)_{|\Sigma'} 
	=  \{ \omega' \in \Omega_{\Sigma'} \mid \text{there is } \omega \in \lb \phi \rb_\Sigma \text{ s.t.\ } \omega \sat_\Sigma \omega' \}.
	\end{align*}
	The \emph{expansion to $\Sigma$} of models $\lb \phi' \rb_{\Sigma'}$ is defined as
	\begin{equation*}
	(\lb \phi' \rb_{\Sigma'})_{\uparrow \Sigma} = \underset{\omega' \in \lb \phi' \rb_{\Sigma'}}{\bigcup} \omega'_{\uparrow \Sigma},
	\end{equation*}
	where $\omega'_{\uparrow \Sigma} = \{ \omega \in \Omega_{\Sigma} \mid \omega \models_\Sigma \omega' \}$.
	Thereby, $\omega \sat_{\Sigma} \omega'$ denotes that $\omega \in \Omega_{\Sigma}$ is more specific than $\omega' \in \Omega_{\Sigma'}$ w.r.t. $\Sigma$, which holds if and only if $\omega_{|\Sigma'} = \omega'$.
\end{definition}

Notice that multiple subsequently performed reductions $(\lb \phi \rb_{|\Sigma'})_{|\Sigma''}$ can be reduced to a single reduction $\lb \phi \rb_{|\Sigma''}$, if the signature $\Sigma''$ is a subset of $\Sigma'$. %

In this work, we generally argue about epistemic states in the form of ordinal conditional functions (OCFs) introduced in a more general form by Spohn \shortcite{spohn1988ordinal}. An OCF $\kappa$ is a ranking function that assigns a rank $r \in \mathbb{N}_0$ to each interpretation $\omega \in \Omega_{\Sigma}$ with $\kappa^{-1}(0) \neq \emptyset$. The rank of an interpretation can be understood as a degree of plausibility, where $\kappa(\omega) = 0$ means that $\omega$ is most plausible. The most plausible interpretations according to an OCF $\kappa$ are also called models of $\kappa$, and are therefore denoted by $\lb \kappa \rb_\Sigma$. The rank of formula $\kappa(\phi) = \min \{ \kappa(\omega) \mid \omega \in \lb \phi \rb \}$ is given by the minimal rank of its models, where $\kappa(\phi \vee \psi) = \min \{ \kappa(\phi), \kappa(\psi) \}$. The beliefs of an OCF $Bel_\Sigma(\kappa) = \{ \phi \in \mathcal{L}_\Sigma \mid \lb \kappa \rb \sat \phi \}$ is the deductively closed set of formulas $\phi \in \mathcal{L}_\Sigma$ that are satisfied by the OCF's models $\lb \kappa \rb_\Sigma$. Instead of $Bel_\Sigma(\kappa) \sat \phi$, we also write $\kappa \sat \phi$.

\section{Delgrande's Forgetting and Marginalisation}\label{sec:forgetting}
In this section, we will first introduce Delgrande's general forgetting approach \cite{delgrande2017knowledge} as well as some of its most important properties. Afterwards, we consider the OCF marginalisation as a kind of forgetting \cite{beierle2019towards} and show that it generalizes Delgrande's definition to epistemic states.

\subsection{Delgrande's General Forgetting Approach}\label{sec:delgrande}
In \cite{delgrande2017knowledge}, Delgrande defines a general forgetting approach with the goal to unify many of the hitherto existing logic-specific forgetting definitions, e.g. forgetting in propositional logic \cite{boole1854investigation}, first-order logic \cite{lin1994forget}, or answer set programming \cite{wong2009forgetting,zhang2006solving}.
While most of these logic-specific approaches depend on the syntactical structure of the knowledge, Delgrande defines forgetting on the knowledge level itself, which means that it is independent of any syntactical properties, and only argues about the beliefs that can be inferred. Concretely, this is realized by arguing about the deductive closure $Cn_\Sigma(\Gamma)$ of a set of formulas $\Gamma$ as seen in \Cref{def:forget_delgrande}

\begin{definition}\label{def:forget_delgrande}\textnormal{\cite{delgrande2017knowledge}}
	Let $\Sigma$ and $P$ be signatures, $\mathcal{L}_{\Sigma}$ a language with corresponding consequence operator $Cn_\Sigma$, and $\mathcal{L}_{\Sigma \setminus P} \subseteq \mathcal{L}_{\Sigma}$ a sub-language, then \emph{forgetting a signature} $P$ in a set of formulas $\Gamma \subseteq \mathcal{L}_\Sigma$ is defined as
	\begin{equation*}
	\mathcal{F}(\Gamma, P) = Cn_\Sigma (\Gamma) \cap \mathcal{L}_{\Sigma \setminus P}.
	\end{equation*}
\end{definition}

By intersecting the prior knowledge $Cn_\Sigma(\Gamma)$ with the sub-language $\mathcal{L}_{\Sigma \setminus P}$ all formulas that mention any signature element $\rho \in P$ will be removed. Therefore, forgetting according to \Cref{def:forget_delgrande} results in those consequences of $\Gamma$ that are included in the reduced language $\mathcal{L}_{\Sigma \setminus P}$.
However, since many of the logic-specific forgetting approaches do not result in a sub-language, Delgrande provides a second definition of forgetting that results in the original language instead (\Cref{def:forget_delgrande_orig_signature}). This allows comparing the results of the different forgetting approaches more easily.

\begin{definition}\label{def:forget_delgrande_orig_signature}\textnormal{\cite{delgrande2017knowledge}}
	Let $\Sigma$ and $P$ be signatures and $\mathcal{L}_\Sigma$ a language with corresponding consequence operator $Cn_\Sigma$, then \emph{forgetting a signature $P$ in the original language $\mathcal{L}_\Sigma$} in a set of formulas $\Gamma \subseteq \mathcal{L}_\Sigma$ is defined as
	\begin{equation*}
	\mathcal{F}_O(\Gamma, P) = Cn_\Sigma (\mathcal{F}(\Gamma, P)).
	\end{equation*}
\end{definition}

Thereby, forgetting in the original language $\mathcal{L}_\Sigma$ is defined as the deductive closure of $\mathcal{F}(\Gamma, P)$ with respect to $\Sigma$. Due to the syntax independent nature of Delgrande's forgetting definition, it is theoretically applicable to each logic with a well-defined consequence operator. Note that even though the posterior knowledge still consists of formulas mentioning the forgotten signature elements $P$, we know that they do not provide any information about $P$, since forgetting in the original signature results in knowledge equivalent the result of forgetting in the reduced language, due to the deductive closure $Cn_\Sigma$.
This also follows from the model theoretical properties of both forgetting definitions stated in \Cref{th:delgrande_models_signature}.

\begin{theorem}\textnormal{\cite{delgrande2017knowledge}}\label{th:delgrande_models_signature}
	Let $\Gamma \subseteq \mathcal{L}_\Sigma$ be a set of formulas and P a signature, then the following equations hold:
	\begin{enumerate}
		\item $\lb \mathcal{F}(\Gamma, P) \rb_{\Sigma \setminus P} = (\lb \Gamma \rb_{\Sigma})_{|(\Sigma \setminus P)}$
		\item $\lb \mathcal{F}(\Gamma, P) \rb_{\Sigma} = ((\lb \Gamma \rb_\Sigma)_{|(\Sigma \setminus P)})_{\uparrow\Sigma}$
	\end{enumerate}
\end{theorem}

From \Cref{th:delgrande_models_signature}, we can conclude that the models of forgetting in the original language are equal to those of forgetting in the reduced language with respect to $\Sigma$ (\Cref{cor:models_delgrande_forgetting_orig}).

\begin{corollary}\label{cor:models_delgrande_forgetting_orig}
	Let $\Gamma \subseteq \mathcal{L}_\Sigma$ be a set of formulas and $P$ a signature, then the following holds:
	\begin{equation*}
	\lb \mathcal{F}_O(\Gamma, P) \rb_\Sigma = (\lb \mathcal{F}(\Gamma, P) \rb_{\Sigma \setminus P})_{\uparrow \Sigma} = \lb \mathcal{F}(\Gamma, P) \rb_\Sigma
	\end{equation*}
\end{corollary}

In \Cref{ex:delgrande_example} below, we illustrate the relations of both forgetting definitions stated by Delgrande.

\begin{example}\label{ex:delgrande_example}
	In this example, we illustrate both Delgrande's forgetting in the reduced as well as in the original language, and its effects on the model level. For this, we consider the knowledge base $\Gamma = \{ p \rightarrow b, f \rightarrow \overline{p}, f \rightarrow b, \overline{f} \rightarrow (p \vee \overline{b}) \} \subseteq \mathcal{L}_\Sigma$ with $\Sigma = \{ p, b, f \}$, where the signature elements can be read as:
	\begin{align*}
		p &- \text{the observed animal is a penguin}, \\
		b &- \text{the observed animal is a bird}, \\
		f &- \text{the observed animal can fly}.
	\end{align*}
	Thus, $\overline{f} \rightarrow (p \vee \overline{b})$ for example reads \emph{if the observed animal cannot fly, then it is a penguin or not a bird at all}. In the following, we want to forget the subsignature $\{ p \} \subseteq \Sigma$. Forgetting $\{p\}$ in the reduced language $\mathcal{L}_{\Sigma \setminus \{p\}}$ results in
	\begin{equation*}
		\mathcal{F}(\Gamma, \{p\}) = Cn_\Sigma(\Gamma) \cap \mathcal{L}_{\Sigma \setminus \{p\}} = Th_\Sigma(\lb \Gamma \rb_\Sigma) \cap \mathcal{L}_{\Sigma \setminus \{p\}},
	\end{equation*}
	where $\lb \Gamma \rb_\Sigma = \{ \overline{p}\overline{b}\pt\overline{f}, pb\overline{f}, \overline{p}bf \}$. Concretely, $\mathcal{F}(\Gamma, \{p\})$ consists of all conclusions that can be drawn from $\Gamma$ and are part of the reduced language $\mathcal{L}_{\Sigma \setminus \{p\}}$, i.e. those conclusions that do not argue about penguins ($p$).
	According to \Cref{th:delgrande_models_signature}, we know that the models after forgetting $\{p\}$ from $\Gamma$ correspond to the prior models $\lb \Gamma \rb_\Sigma$ reduced to $\Sigma \setminus \{p\}$:
	\begin{align*}
		& ~\lb \mathcal{F}(\Gamma, \{p\}) \rb_{\Sigma \setminus \{p\}} = (\lb \Gamma \rb_\Sigma)_{|\Sigma \setminus \{p\}} \\
		= &~\{ \overline{p}\overline{b}\pt\overline{f}, pb\overline{f}, \overline{p}bf \}_{|\Sigma \setminus \{p\}} = \{ \overline{b}\pt\overline{f}, b\overline{f}, bf \}.
	\end{align*}
	Thus, the posterior models after forgetting $\{p\}$ are obtain by mapping each interpretation $\dot{p}\dot{b}\dot{f}$ to $\dot{b}\dot{f}$.
	
	If we forget $\{p\}$ in the original language $\mathcal{L}_\Sigma$ instead, we obtain
	\begin{align*}
		& ~\mathcal{F}_O(\Gamma, \{p\}) = Cn_\Sigma(\mathcal{F}(\Gamma, \{p\})) = Th(\lb \mathcal{F}(\Gamma, \{p\})\rb_\Sigma) \\
		= &~Th((\lb \mathcal{F}(\Gamma, \{p\})\rb_{\Sigma \setminus \{p\}})_{\uparrow \Sigma}) = Th(((\lb \Gamma \rb_\Sigma)_{|\Sigma \setminus \{p\}})_{\uparrow \Sigma}).
	\end{align*}
	By means of the deductive closure of $\mathcal{F}(\Gamma, \{p\})$ with respect to $\Sigma$, the result of forgetting in the reduced language is extended by those formulas $\phi \in \mathcal{L}_\Sigma$ in the original language that can be inferred by it. However, due to the relations of the prior models $\lb \Gamma \rb_\Sigma$ and those after forgetting $\{p\}$ in the reduced and the original language
	\begin{align*}
		&~\lb \mathcal{F}_O(\Gamma, \{p\}) \rb_\Sigma = ((\lb \Gamma \rb_\Sigma)_{|\Sigma \setminus \{p\}})_{\uparrow \Sigma}\\
		= &~(\{ \overline{p}\overline{b}\pt\overline{f}, pb\overline{f}, \overline{p}bf \}_{|\Sigma \setminus \{p\}})_{\uparrow \Sigma} = \{ \overline{b}\pt\overline{f}, b\overline{f}, bf \}_{\uparrow \Sigma} \\
		=& ~\{ \overline{p}\overline{b}\pt\overline{f}, p\overline{b}\pt\overline{f}, pb\overline{f}, \overline{p}b\overline{f}, \overline{p}bf, pbf \},
	\end{align*}
	we see that $\mathcal{F}_O(\Gamma, \{p\})$ can only contain trivial proposition about penguins ($p$), since we know that if $p\dot{b}\dot{f} \in \lb \mathcal{F}_O(\Gamma, \{p\}) \rb$, then $\overline{p}\dot{b}\dot{f} \in \lb \mathcal{F}_O(\Gamma, \{p\}) \rb$ must hold as well. This way non-trivial propositions about penguins are prevented, which is why forgetting in the original language can still be considered as forgetting ${p}$. We provide an overview of the different models in \Cref{tab:delgrande_example}.
	
	\begin{table}
		\centering
		\begin{tabular}{| c | c | c |}
			\hline
			$\lb \Gamma \rb_\Sigma$ & $\lb \mathcal{F}(\Gamma, \{p\}) \rb_{\Sigma \setminus \{p\}}$ & $\lb \mathcal{F}_O(\Gamma, \{p\}) \rb_\Sigma$ \\ \hline \hline
			\multirow{2}{*}{$\overline{p}\overline{b}\pt\overline{f}$, $pb\overline{f}$, $\overline{p}bf$} & \multirow{2}{*}{$\overline{b}\pt\overline{f}$, $b\overline{f}$, $bf$} & $\overline{p}\overline{b}\pt\overline{f}$, $p\overline{b}\pt\overline{f}$, $\overline{p}b\overline{f}$,\\
			&& $pb\overline{f}$, $\overline{p}bf$, $pbf$ \\ \hline
		\end{tabular}
		\caption{
			Models of $\Gamma$, $\mathcal{F}(\Gamma, \{p\})$, and $\mathcal{F}_O(\Gamma, \{p\})$ with respect to the corresponding signatures of the languages, where
			$\Gamma = \{ p \rightarrow b, f \rightarrow \overline{p}, f \rightarrow b, \overline{f} \rightarrow (p \vee \overline{b}) \} \subseteq \mathcal{L}_\Sigma$ and $\Sigma = \{p,b,f\}$.}
		\label{tab:delgrande_example}
	\end{table}
\end{example}

Besides defining a general forgetting approach, Delgrande also states several properties of his definition, which he refers to as \emph{right} and \emph{desirable} \cite{delgrande2017knowledge}. In this work, we refer to these properties as \dfp{1}-\dfp{7} as stated in \Cref{th:delgrande_postulates}.

\begin{theorem}\label{th:delgrande_postulates}\textnormal{\cite{delgrande2017knowledge}}
	Let $\mathcal{L}_\Sigma$ be a language over signature $\Sigma$ and $Cn_\Sigma$ the corresponding consequence operator, then the following relations hold for all sets of formulas $\Gamma, \Gamma' \subseteq \mathcal{L}_\Sigma$ and signatures $P, P'$.
	\begin{description}
		\item[(DFP-1)] $\Gamma \sat \mathcal{F}(\Gamma, P)$
		\item[(DFP-2)] If $\Gamma \sat \Gamma'$, then $\mathcal{F}(\Gamma, P) \sat \mathcal{F}(\Gamma', P)$
		\item[(DFP-3)] $\mathcal{F}(\Gamma, P) = Cn_{\Sigma \setminus P}(\mathcal{F}(\Gamma, P))$
		\item[(DFP-4)] If $P' \subseteq P$, then $\mathcal{F}(\Gamma, P) = \mathcal{F}(\mathcal{F}(\Gamma, P'), P)$
		\item[(DFP-5)] $\mathcal{F}(\Gamma, P \cup P') = \mathcal{F}(\Gamma, P) \cap \mathcal{F}(\Gamma, P')$
		\item[(DFP-6)] $\mathcal{F}(\Gamma, P \cup P') = \mathcal{F}(\mathcal{F}(\Gamma, P), P')$
		\item[(DFP-7)] $\mathcal{F}(\Gamma, P) = \mathcal{F}_O(\Gamma, P) \cap \mathcal{L}_{\Sigma \setminus P}$
	\end{description}
\end{theorem}

\dfp{1} states the monotony of forgetting, which means that it is not possible to obtain new knowledge by means of forgetting.
\dfp{2} states that any consequence relation $\Gamma \sat \Gamma'$ of prior knowledge sets is preserved after forgetting a signature $P$ in both. \dfp{3} describes that forgetting always results in a deductively closed knowledge set with respect to the reduced signature. This also corresponds to Delgrande's idea of defining forgetting on the knowledge level -- forgetting is applied to a deductively closed set and results in such. In \dfp{4}, Delgrande states that forgetting two signatures $P'$ and $P$ consecutively always equals the forgetting of $P$, if $P'$ is included in $P$. Thus, forgetting a signature twice has no effect on the prior knowledge. \dfp{5} and \dfp{6} argue about iterative and simultaneous forgetting. Finally, \dfp{7} describes the relation between forgetting in the original and the reduced language by stating that the result of forgetting in the reduced language can always be obtained by intersecting the result of forgetting in the original language with the reduced language. Note that we changed the notation of \dfp{7} in order to make it more explicit. For more information on \dfp{1}-\dfp{7} we refer to \cite{delgrande2017knowledge}.

\subsection{Marginalisation}\label{sec:marginalisation}
A general framework of forgetting and its instantiation to an approach using OCFs is developed in \cite{beierle2019towards}. For the purpose of this paper, we concentrate on the marginalisation, which on a cognitive level corresponds to the notion of focussing and can briefly be summarized as:
\begin{enumerate}
	\item Focussing on relevant aspects retains our beliefs about them.
	\item Focussing on relevant aspects (temporarily) changes our beliefs such that they do not contain any information about irrelevant aspects anymore.
\end{enumerate}

In practice, this notion of forgetting is useful when it comes to efficient and focussed query answering by means of abstracting from irrelevant details, e.g. marginalisation is crucially used in all inference techniques for probabilistic networks.
At this point, we consider the relevant aspects to be given and focus on the marginalisation (\Cref{def:marginalization_rank}) as a kind of forgetting as such.
\begin{definition}\cite{beierle2019towards}\label{def:marginalization_rank}
	Let $\kappa$ be an OCF over signature $\Sigma$ and $\omega' \in {\Omega_{\Sigma'}}$ an interpretation with $\Sigma' \subseteq \Sigma$. $\krs$ is called a \emph{marginalisation} of $\kappa$ to $\Sigma'$ with
	\begin{equation*}
		\krs(\omega') = \min\{ \kappa(\omega) \mid \omega \in \Omega_\Sigma \text{ with } \omega \sat \omega' \}.
	\end{equation*}
\end{definition}
By marginalising an OCF to a subsignature $\Sigma'$, we consider interpretations over $\Sigma'$ as conjunctions and assign the corresponding rank to them.

The first notion of focussing corresponds to \Cref{lem:marginalization_equiv}, which states that a formula over the reduced signature is believed after the marginalisation, if and only if it is also believed by the prior OCF. Thus, the beliefs that only argue about the relevant aspects $\Sigma'$ are retained.
\begin{lemma}\label{lem:marginalization_equiv}
	Let $\kappa$ be an OCF over $\Sigma$ and $\Sigma' \subseteq \Sigma$, then for each $\phi \in \mathcal{L}_{\Sigma'}$ the following holds:
	\begin{equation*}
	\kappa_{|\Sigma'} \sat \phi \Leftrightarrow \kappa \sat \phi 
	\end{equation*}
\end{lemma}
Similarly to Delgrande's forgetting, marginalisation reduces beliefs to a subsignature.
Note that \Cref{lem:marginalization_equiv} directly follows from \cite{beierle2019towards}, where they already stated that this relations generally holds for conditional beliefs. Furthermore, \Cref{lem:marginalization_equiv} allows us to express the posterior beliefs analogously to Delgrande's forgetting definition (\Cref{prop:forgetting_belief_1}).

\begin{proposition}\label{prop:forgetting_belief_1}
	Let $\kappa$ be an OCF over signature $\Sigma$ and $\Sigma' \subseteq \Sigma$ a reduced signature.
	\begin{equation*}
	Bel(\kappa_{|\Sigma'}) = Bel(\kappa) \cap \mathcal{L}_{\Sigma'}
	\end{equation*}
\end{proposition}

\begin{proof}[Proof of \Cref{prop:forgetting_belief_1}]
Due to Lemma \ref{lem:marginalization_equiv}, we have $	Bel(\kappa) \cap \mathcal{L}_{\Sigma'} = Bel(\kappa_{|\Sigma'}) \cap \mathcal{L}_{\Sigma'}
	= Bel(\kappa_{|\Sigma'})$ because $(Bel(\kappa_{|\Sigma'}) \subseteq \mathcal{L}_{\Sigma'})$.
\end{proof}
Thereby, \Cref{prop:forgetting_belief_1} also corresponds to the second notion of focussing, due to the intersection with reduced language $\mathcal{L}_{\Sigma'}$.
The above-stated relations of the prior and posterior beliefs further imply that the models of the posterior beliefs are equal to the those of the prior when reducing them to $\Sigma'$ (\Cref{prop:ocf_models_marginalization}). This rather technical property allows us to freely switch between the models of the marginalised and the prior OCF, which will be useful in later proofs. 

\begin{proposition}\label{prop:ocf_models_marginalization}
	Let $\kappa$ be an OCF over signature $\Sigma$ and $\Sigma' \subseteq \Sigma$ a subsignature.
	Then $\lb \krs \rb = \lb \kappa \rb_{|\Sigma'}$ holds.
\end{proposition}

\begin{proof}[Proof of \Cref{prop:ocf_models_marginalization}]
 By definition,
 \begin{equation*}
 	\lb \krs \rb 	= \{ \omega' \in \Omega_{\Sigma'} \mid \krs(\omega') = 0 \}, 
 \end{equation*}
 so applying \Cref{def:marginalization_rank} yields
 \begin{align*}
 	& \lb \krs \rb \\
 	=& \{ \omega' \in \Omega_{\Sigma'} \mid \min \{ \kappa(\omega) \mid \omega \in \Omega_{\Sigma} \twith \omega \sat \omega' \} = 0 \},
 \end{align*}
  which is the same as
  \begin{align*}
  	 & \{ \omega' \in \Omega_{\Sigma'} \mid \exists \omega \in \Omega_{\Sigma} \twith  \omega \sat \omega' \tand \kappa(\omega) = 0 \}\\
  	 = & \{ \omega' \in \Omega_{\Sigma'} \mid \exists \omega \in \Omega_{\Sigma} \twith  \omega \sat \omega' \tand \omega \in \lb \kappa \rb \} \\
  	 = & \{ \omega' \in \Omega_{\Sigma'} \mid \exists \omega \in \lb \kappa \rb \twith \omega \sat \omega' \} = \lb \kappa \rb_{|\Sigma'}.
  \end{align*}
\end{proof}

Similar to Delgrande's idea of forgetting in the original language, we might be interested in arguing about the original signature after focussing, e.g. for reasons of comparability. Thus, we define the concept of lifting an OCF in \Cref{def:marginalization_rank_orig} below.
\begin{definition}\label{def:marginalization_rank_orig}
	Let $\kappa'$ be an OCF over signature $\Sigma' \subseteq \Sigma$. A \emph{lifting} of $\kappa'$ to $\Sigma$, denoted by $\kappa'_{\uparrow \Sigma}$, is uniquely defined by $\kappa'_{\uparrow \Sigma}(\omega) = \kappa'(\omega_{|\Sigma'})$ for all $\omega \in \Omega_{\Sigma}$.
\end{definition}
By means of lifting an OCF $\kappa'$ over signature $\Sigma'$ to a signature $\Sigma$ with $\Sigma' \subseteq \Sigma$, we (re-)introduce new signature elements to $\kappa'$ in a way that $\kappa'_{\uparrow \Sigma}$ acts invariantly towards them. This is guaranteed by the fact that all interpretations $\omega \in \Omega_{\Sigma}$ that only differ in the truth value they assign to the new signature elements $\Sigma \setminus \Sigma'$ are assigned to the same rank.
Analogously to \Cref{prop:ocf_models_marginalization}, we show in \Cref{prop:ocf_models_lifting} that the models of a lifted OCF are equal to the prior models when expanded to the super-signature.

\begin{proposition}\label{prop:ocf_models_lifting}
	Let $\kappa'$ be an OCF over signature $\Sigma' \subseteq \Sigma$.
	Then the models of the lifted $\kappa'$ are the expanded models of $\kappa'$, i.e., $\lb \kappa'_{\uparrow \Sigma} \rb = \lb \kappa' \rb_{\uparrow \Sigma}$.
\end{proposition}

\begin{proof}[Proof of \Cref{prop:ocf_models_lifting}]
By definition,
\begin{equation*}
	\lb \kappa' \rb_{\uparrow \Sigma} = \underset{\omega' \in \lb \kappa' \rb }{\bigcup} \{ \omega \in \Omega_\Sigma \mid \omega \sat \omega' \},
\end{equation*}
and hence
\begin{align*}
	\lb \kappa' \rb_{\uparrow \Sigma} &= \{ \omega \in \Omega_\Sigma \mid  \exists \omega' \in \lb \kappa' \rb_{\Sigma'} \twith \omega \sat \omega' \} \\
	&= \{ \omega \in \Omega_{\Sigma} \mid  \exists \omega' \in \lb \kappa' \rb_{\Sigma'} \twith \omega_{|\Sigma'} \equiv \omega' \}, 
\end{align*}
due to $\omega \sat \omega' \Leftrightarrow \omega_{|\Sigma'} = \omega'$ (\Cref{def:reduct_extension}). 
Since we know that if there is an interpretation $\omega' \in \lb \kappa' \rb_{\Sigma'}$ that is equivalent to $\omega_{|\Sigma'}$, then $\omega_{|\Sigma'}$ is included in $\lb \kappa' \rb_{\Sigma'}$ as well, and vice-versa, this last set is the same as 
\begin{align*}
	& \{ \omega \in \Omega_{\Sigma} \mid \omega_{|\Sigma'} \in \lb \kappa' \rb_{\Sigma'} \} = \{ \omega \in \Omega_{\Sigma} \mid \kappa'(\omega_{|\Sigma'}) = 0  \} \\ 
	= & \{ \omega \in \Omega_{\Sigma} \mid \kappa'_{\uparrow \Sigma}(\omega) = 0  \} = \lb \kappa'_{\uparrow \Sigma} \rb,
\end{align*}
again by definition. 
\end{proof}

Therefore, we also know that the beliefs after lifting are equivalent to the prior with respect to $\Sigma$, which can also be denoted as the deductive closure of the prior beliefs with respect to $\Sigma$ (\Cref{prop:lifting_belief}).

\begin{proposition}\label{prop:lifting_belief}
	Let $\kappa'$ be an OCF over signature $\Sigma' \subseteq \Sigma$ and $\kappa'_{\uparrow \Sigma}$ be a lifting of $\kappa'$ to $\Sigma$, then the beliefs of $\kappa'_{\uparrow \Sigma}$ are given by
	$Bel(\kappa'_{\uparrow \Sigma}) = Cn_\Sigma(Bel(\kappa'))$.
\end{proposition}

\begin{proof}[Proof of \Cref{prop:lifting_belief}]
In a straightforward way, we obtain from \Cref{prop:ocf_models_lifting}
\begin{align*}
	&~Bel(\kappa'_{\uparrow \Sigma}) =  Th(\lb \kappa'_{\uparrow \Sigma} \rb) \\
	= &~ Cn_\Sigma(\underset{\omega \in \lb \kappa'_{\uparrow \Sigma} \rb}{\bigvee} \omega) = Cn_\Sigma(\underset{\omega \in \lb \kappa' \rb_{\uparrow \Sigma}}{\bigvee} \omega) \\
	= &~ Cn_\Sigma(\underset{\omega \in \underset{\omega' \in \lb \kappa' \rb}{\bigcup} \omega'_{\uparrow \Sigma}}{\bigvee} \omega) = Cn_\Sigma(\underset{\omega' \in \lb \kappa' \rb}{\bigvee} (\underset{\omega \in \omega'_{\uparrow \Sigma}}{\bigvee} \omega)) \\ 
	= &~ Cn_\Sigma(\underset{\omega' \in \lb \kappa' \rb}{\bigvee} \omega') = Cn_\Sigma(Cn_{\Sigma'}(\underset{\omega' \in \lb \kappa' \rb}{\bigvee} \omega')) \\
	= &~ Cn_\Sigma(Th(\lb \kappa' \rb)) = Cn_\Sigma(Bel(\kappa')).
\end{align*}
\end{proof}

\Cref{prop:lifting_belief} clearly shows that the beliefs of a marginalised OCF relate to those after lifting it to the original signature again in the same way Delgrande's forgetting in the original language relates to forgetting in the reduced language (see \Cref{def:forget_delgrande_orig_signature}).

Finally, we can show that the marginalisation generalizes Delgrande's forgetting definition to epistemic states, since both forgetting approaches result in equivalent posterior beliefs when applied to the same prior knowledge (\Cref{th:marginalization_delgrande_equiv}).
\begin{theorem}\label{th:marginalization_delgrande_equiv}
	Let $\Gamma \subseteq \mathcal{L}_\Sigma$ be a set of formulas and $\kappa$ an OCF over signature $\Sigma$ with $Bel(\kappa) \equiv \Gamma$, then
	\begin{equation*}
	\mathcal{F}(\Gamma, P) = Bel(\kappa_{|(\Sigma \setminus P)})
	\end{equation*}
	holds for each signature $P$.
\end{theorem}

\begin{proof}[Proof of \Cref{th:marginalization_delgrande_equiv}]
Due to \Cref{prop:forgetting_belief_1}, we have $Bel(\kappa_{|(\Sigma \setminus P)}) = Bel(\kappa) \cap \mathcal{L}_{\Sigma \setminus P} $. 
Since $Bel(\kappa) \equiv \Gamma$, this is the same as $Cn_\Sigma(\Gamma) \cap \mathcal{L}_{\Sigma \setminus P} = \mathcal{F}(\Gamma, P)$, by definition. 
\end{proof}

The equivalence of the prior knowledge for both approaches can be stated as $Bel(\kappa) \equiv \Gamma$, which means that the set of formulas Delgrande's forgetting is applied to must be equivalent to the prior beliefs $Bel(\kappa)$. Furthermore, note that Delgrande's forgetting definition argues about the elements that should be forgotten, while the marginalisation argues about the remaining subsignature.

In \Cref{ex:marginalisation_lifting_example} below, we illustrate the marginalisation as well as a subsequently performed lifting of an OCF $\kappa$ over the signature $\Sigma = \{p,b,f\}$, and show how marginalisation and lifting corresponds to Delgrande's forgetting definitions. For this we refer to the example on Delgrande's forgetting (\Cref{ex:delgrande_example}).

\begin{example}\label{ex:marginalisation_lifting_example}
	In this example, we illustrate a marginalisation and a consecutively performed lifting of the OCF $\kappa$ over $\Sigma = \{p, b, f\}$ (see \Cref{ex:delgrande_example}) given in \Cref{tab:marginalisation_lifting_example}, as well as the relations to Delgrande's forgetting definitions. In the following, we want to forget the subsignature $\{p\} \subseteq \Sigma$.
	
	First of all, we want to note that the beliefs of $\kappa$ are equivalent to the knowledge base $\Gamma$ (\Cref{ex:delgrande_example}), since their corresponding models are the same:
	\begin{align*}
		& ~Bel_\Sigma(\kappa) = Th(\lb \kappa \rb_\Sigma) = Th(\{ \overline{p}\overline{b}\pt\overline{f}, pb\overline{f}, \overline{p}bf \}) \\
		= & ~Th(\lb \Gamma \rb_\Sigma) = Cn_\Sigma(\Gamma) \equiv \Gamma
	\end{align*}
	
	Marginalising $\kappa$ to $\Sigma \setminus P$ results in $\kappa_{|(\Sigma \setminus P)}$ as given in \Cref{tab:marginalisation_lifting_example}. There it can be seen that the posterior most plausible interpretation correspond to those of $\kappa$ when omitting $p$, i.e. each interpretation $\dot{p}\dot{b}\dot{f} \in \lb \kappa \rb$ is mapped to $\dot{b}\dot{f} \in \lb \kappa_{|(\Sigma \setminus \{p\})} \rb$. This exactly corresponds to the way Delgrande's forgetting in the reduced language affects the models of the given knowledge base $\Gamma$:
	\begin{align*}
		& ~\lb \kappa_{|(\Sigma \setminus \{p\})} \rb_{\Sigma \setminus P} = \lb \kappa \rb_{|(\Sigma \setminus \{p\})} = \{ \overline{p}\overline{b}\pt\overline{f}, pb\overline{f}, \overline{p}bf \}_{|(\Sigma \setminus \{p\})} \\
		=& ~\{ \overline{b}\pt\overline{f}, b\overline{f}, bf \} = \lb \mathcal{F}(\Gamma, \{p\}) \rb_{\Sigma \setminus \{p\}}
	\end{align*}
	In conclusion, we that know the posterior beliefs of the marginalisation and the result of Delgrande's forgetting must be equal:
	\begin{align*}
		& ~Bel(\kappa_{|(\Sigma \setminus \{p\})}) = Th(\lb \kappa_{|(\Sigma \setminus \{p\})} \rb_{\Sigma \setminus \{p\}}) \\
		= & ~ Th(\{ \overline{b}\pt\overline{f}, b\overline{f}, bf \})
		= Th(\lb \mathcal{F}(\Gamma, \{p\}) \rb_{\Sigma \setminus P}) = \mathcal{F}(\Gamma, \{p\})
	\end{align*}
	
	When we lift the marginalised OCF $\kappa_{|(\Sigma \setminus \{p\})}$ back to the original signature $\Sigma$, the posterior most plausible interpretations can be obtained by mapping each interpretation $\dot{b}\dot{f} \in \lb \kappa_{|(\Sigma \setminus \{p\})} \rb_{\Sigma \setminus \{p\}}$ to $\{ p\dot{b}\dot{f}, \overline{p}\dot{b}\dot{f} \} \subseteq \lb (\kappa_{|(\Sigma \setminus \{p\})})_{\uparrow \Sigma} \rb_\Sigma$ (see \Cref{tab:marginalisation_lifting_example}). Just as for the marginalisation, this exactly corresponds to the way Delgrande's forgetting in the original language affects the prior models of the knowledge base $\Gamma$:
	\begin{align*}
		& ~\lb (\kappa_{|(\Sigma \setminus \{p\}) })_{\uparrow \Sigma} \rb_\Sigma = \{ \overline{b}\pt\overline{f}, b\overline{f}, bf \}_{\uparrow \Sigma} \\
		=& ~ \{ \overline{p}\overline{b}\pt\overline{f}, p\overline{b}\pt\overline{f}, pb\overline{f}, \overline{p}b\overline{f}, \overline{p}bf, pbf \} = \lb \mathcal{F}_O(\Gamma, \{p\}) \rb_\Sigma
	\end{align*}
	Therefore, the result of Delgrande's forgetting in the original language is equal to the beliefs after marginalising and lifting $\kappa$:
	\begin{align*}
		& ~Bel_\Sigma((\kappa_{|(\Sigma \setminus \{p\}) })_{\uparrow \Sigma}) = Th(\lb (\kappa_{|(\Sigma \setminus \{p\}) })_{\uparrow \Sigma} \rb_\Sigma) \\
		=& ~ Th(\{ \overline{p}\overline{b}\pt\overline{f}, p\overline{b}\pt\overline{f}, pb\overline{f}, \overline{p}b\overline{f}, \overline{p}bf, pbf \}) \\
		=& ~ Th(\lb \mathcal{F}_O(\Gamma, \{p\}) \rb_\Sigma) = \mathcal{F}_O(\Gamma, \{p\})
	\end{align*}
	
	\begin{table}
		\centering
		\begin{tabular}{| c | c | c | c |}
			\hline
			& $\kappa$ & $\kappa_{|(\Sigma \setminus \{p\})}$ & $(\kappa_{|(\Sigma \setminus \{p\})})_{\uparrow \Sigma}$ \\ \hline\hline
			2 & $pbf$, $\overline{p}b\overline{f}$ & - & - \\ \hline
			1 & $p\overline{b}\pt\overline{f}$, $\overline{p}\overline{b}f$, $p\overline{b}f$ & $\overline{b}f$ & $p\overline{b}f$, $\overline{p}\overline{b}f$ \\ \hline
			\multirow{2}{*}{0} & \multirow{2}{*}{$\overline{p}\overline{b}\pt\overline{f}$, $pb\overline{f}$, $\overline{p}bf$} & \multirow{2}{*}{$\overline{b}\pt\overline{f}$, $b\overline{f}$, $bf$} & $\overline{p}\overline{b}\pt\overline{f}$, $p\overline{b}\pt\overline{f}$, $pb\overline{f}$,\\
			&&& $\overline{p}b\overline{f}$, $\overline{p}bf$, $pbf$ \\ \hline
		\end{tabular}
		\caption{OCFs $\kappa$ over signature $\Sigma = \{p,b,f\}$, as well as its marginalisation $\kappa_{|(\Sigma \setminus \{p\}) }$ and the corresponding lifting $(\kappa_{|(\Sigma \setminus \{p\}) })_{\uparrow \Sigma}$.}
		\label{tab:marginalisation_lifting_example}
	\end{table}
\end{example}

From the equivalence stated in \Cref{th:marginalization_delgrande_equiv}, we know that all relations of the logic-specific forgetting approaches and Delgrande's general approach that can be traced back to the equivalence of the results must hold for the marginalisation as well. In the following, we exemplarily state this for Boole's atom forgetting in propositional (\Cref{def:forget_boole}), of which we know that it can also be described by means of $\mathcal{F}$ (\Cref{th:boole_equiv_delgrande}).

\begin{definition}\label{def:forget_boole}\textnormal{\cite{boole1854investigation}}
	Let $\phi \in \mathcal{L}_\Sigma$ be a formula and $\rho \in \mathcal{L}_\Sigma$ be an atom. \emph{Forgetting} $\rho$ in $\phi$ is then defined as
	\begin{equation*}
	forget(\phi, \rho) = \phi[\rho / \top] \vee \phi[\rho / \bot],
	\end{equation*}
	where $\phi[\rho / \top]$ denotes the substitution of $\rho$ by $\top$, and $\phi[\rho / \bot]$ the substitution by $\bot$.
\end{definition}

\begin{theorem}\label{th:boole_equiv_delgrande}\textnormal{\cite{delgrande2017knowledge}}
	Let $\mathcal{L}_\Sigma$ be the language in propositional logic with signature $\Sigma$ and let $\rho \in \Sigma$ be an atom.
	\begin{equation*}
		forget(\phi, \rho) \equiv \mathcal{F}(\phi, \{ \rho \})
	\end{equation*}
\end{theorem}

From \Cref{th:marginalization_delgrande_equiv} and \Cref{th:boole_equiv_delgrande}, we can directly conclude that Boole's forgetting definition can also be realized by means of a marginalisation (\Cref{cor:marginalization_boole_equiv}).

\begin{corollary}\label{cor:marginalization_boole_equiv}
	Let $\kappa$ be an OCF over signature $\Sigma$ and $\phi \in \mathcal{L}_\Sigma$ a formula with $Bel(\kappa) \equiv \phi$, then
	\begin{equation*}
	forget(\phi, \rho) \equiv Bel(\kappa_{|\Sigma \setminus \{ \rho \}})
	\end{equation*}
	holds for each atom $\rho \in \Sigma$.
\end{corollary}

\section{Postulates for Forgetting Signatures in Epistemic States}\label{sec:postulates}
In \cite{delgrande2017knowledge}, Delgrande argues that the properties \dfp{1}-\dfp{7} (\Cref{th:delgrande_postulates}) of his forgetting definition are \emph{right} and \emph{desirable} for describing the general notions of forgetting. Since we already proved that his definition can be generalised to epistemic states by means of the marginalisation, we
also present an extended and generalised form of \dfp{1}-\dfp{7}, namely \dfpessr{1}-\dfpessr{6},
and show that the marginalisation satisfies all of them. For this, let $\Psi, \Phi$ be epistemic states,
$P, P’, P_1, P_2$ signatures, and $\forgetS$ an arbitrary operator that maps an epistemic state
together with a signature to a new epistemic state:
\begin{description}
	\item[\dfpessr{1}] $Bel(\Psi) \sat Bel(\Psi \forgetS P)$
	\item[\dfpessr{2}] If $Bel(\Psi) \sat Bel(\Phi)$, then $Bel(\Psi \forgetS P) \sat Bel(\Phi \forgetS P)$
	\item[\dfpessr{3}] If $P' \subseteq P$, then $Bel((\Psi \forgetS P') \forgetS P) \equiv Bel(\Psi \forgetS P)$
	\item[\dfpessr{4}] $Bel(\Psi \forgetS (P_1 \cup P_2)) \equiv Bel(\Psi \forgetS P_1) \cap Bel(\Psi \forgetS P_2)$
	\item[\dfpessr{5}] $Bel(\Psi \forgetS (P_1 \cup P_2)) \equiv Bel((\Psi \forgetS P_1) \forgetS P_2)$
	\item[\dfpessr{6}] $Bel(\Psi \forgetS P) \equiv Bel((\Psi \forgetS P)_{\uparrow \Sigma}) \cap \mathcal{L}_{\Sigma \setminus P}$
\end{description}

For a detailed explanation of the above-stated postulates \dfpessr{1}-\dfpessr{6}, we refer to the
explanations of the postulates \dfp{1}-\dfp{7} as originally stated by Delgrande. However, there
are a few points we want to emphasise in particular. First, since the beliefs of an epistemic state
are deductively closed by definition, it is not necessary to maintain \dfp{3}. Notice that due to
omitting \dfp{3} the postulates \dfp{4}-\dfp{7} correspond to \dfpessr{3}-\dfpessr{6}.
Furthermore, we expressed the forgetting in the original signature $\mathcal{F}_O(\Gamma, P)$ in
\dfp{7} as the beliefs after forgetting $P$ and lifting the posterior epistemic state back to the
original signature. The models of $\mathcal{F}_O(\Gamma, P)$ are equal to the models of forgetting
$P$ in $\Gamma$ in the reduced signature lifted back to the original signature, i.e. $\lb
\mathcal{F}(\Gamma, P) \rb_{\uparrow \Sigma}$ (\Cref{cor:models_delgrande_forgetting_orig}). When we consider the models of
$Bel((\Psi \circ_f^{\Sigma} P)_{\uparrow \Sigma})$, i.e. $\lb \Psi \circ_f^{\Sigma} P \rb_{\uparrow
 \Sigma}$, we see that this also describes the models after forgetting $P$ lifted back to the original
signature. Therefore, \dfpessr{6} exactly matches the property originally stated by \dfp{7}.
In the following, we refer to those operators satisfying \dfpessr{1}-\dfpessr{6} as signature forgetting operators.

Next, we show in \Cref{th:dfpessr_marginalization} that the marginalisation satisfies \dfpessr{1}-\dfpessr{6}, and therefore not only yields results equivalent to those of Delgrande’s forgetting definition, but also corresponds to
the notions of forgetting stated by Delgrande by means of \dfp{1}-\dfp{7}.

Note that there exist forgetting approaches that yield results semantically equivalent to those of Delgrande's approach, but do not satisfy \dfp{1}-\dfp{7}. An example is Boole's atom forgetting (\Cref{def:forget_boole}), which violates \dfp{3}.

\begin{theorem}\label{th:dfpessr_marginalization}
	Let $\kappa$ be an OCF over signature $\Sigma$ and $P$ a signature. The marginalisation $\kappa_{|(\Sigma \setminus P)}$ to a subsignature $(\Sigma \setminus P) \subseteq \Sigma$ satisfies \dfpessr{1}-\dfpessr{6}.
\end{theorem}

\begin{proof}[Proof of \Cref{th:dfpessr_marginalization}]
	In the following, we assume the epistemic states $\Psi$ and $\Phi$ to be OCFs, since the marginalisation is specifically defined over OCFs, denoted as $\kappa$ and $\kappa'$, and further denote the marginalisation $\kappa_{|\Sigma \setminus P}$ as $\kappa \forgetSm P$.

For \dfpessr{1}, we need to show $Bel(\kappa) \sat Bel(\kappa \forgetSm P)$, which means $  Bel(\kappa) \sat Bel(\kappa_{|(\Sigma \setminus P)})$. This holds due to  \Cref{lem:marginalization_equiv}.
For \dfpessr{2}, we presuppose $Bel(\kappa) \sat Bel(\kappa')$. Then also  $Bel(\kappa) \cap \mathcal{L}_{\Sigma \setminus P} \sat Bel(\kappa') \cap \mathcal{L}_{\Sigma \setminus P}$ which is equivalent to $Bel(\kappa_{|(\Sigma \setminus P)}) \sat Bel(\kappa'_{|(\Sigma \setminus P)})$ because of \Cref{prop:forgetting_belief_1}, and hence by definition, $Bel(\kappa \forgetSm P) \sat Bel(\kappa' \forgetSm P)$. 
	
Regarding \dfpessr{3}, we have the following equalities due to \Cref{prop:ocf_models_marginalization}, and because of $P' \subseteq P$: 
\begin{align*}
	& ~Bel((\kappa \forgetSm P') \forgetSm P) \\
	= & ~Bel(\kappa_{\Sigma \setminus P'} \forgetSm P) =  Bel((\kappa_{|\Sigma \setminus P'})_{|(\Sigma \setminus P') \setminus P}) \\ 
	= &~ Bel((\kappa_{|\Sigma \setminus P'})_{|(\Sigma \setminus (P' \cup P))}) =  Th(\lb (\kappa_{|\Sigma \setminus P'})_{|(\Sigma \setminus (P' \cup P))} \rb) \\
	= &~ \{ \phi \in \mathcal{L}_{\Sigma \setminus (P' \cup P)} \mid  \lb (\kappa_{|\Sigma \setminus P'})_{|(\Sigma \setminus (P' \cup P))} \rb \sat \phi \}\\
	= &~ \{ \phi \in \mathcal{L}_{\Sigma \setminus (P' \cup P)} \mid  \lb \kappa_{|\Sigma \setminus P'} \rb_{|\Sigma \setminus (P' \cup P)} \sat \phi \}\\
	= &~ \{ \phi \in \mathcal{L}_{\Sigma \setminus (P' \cup P)} \mid  (\lb \kappa \rb_{|\Sigma \setminus P'})_{|\Sigma \setminus (P' \cup P)} \sat \phi \} \\
	= &~ \{ \phi \in \mathcal{L}_{\Sigma \setminus (P' \cup P)} \mid  \lb \kappa \rb_{|\Sigma \setminus (P' \cup P)} \sat \phi \}\\
	= &~ \{ \phi \in \mathcal{L}_{\Sigma \setminus P} \mid \lb \kappa \rb_{|\Sigma \setminus P} \sat \phi \}  \\
	= &~ \{ \phi \in \mathcal{L}_{\Sigma \setminus P} \mid \lb \kappa_{|\Sigma \setminus P} \rb \sat \phi \} \\
	= &~ Th(\lb \kappa_{|\Sigma \setminus P} \rb) = Bel(\kappa_{|\Sigma \setminus P}) = Bel(\kappa \forgetSm P).
\end{align*}
	
The proof of \dfpessr{4} is mainly based on \Cref{prop:forgetting_belief_1}, here we compute 
\begin{align*}
	& ~ Bel(\kappa \forgetSm (P_1 \cup P_2)) = Bel(\kappa_{|\Sigma \setminus (P_1 \cup P_2)}) \\
	= &~ Bel(\kappa) \cap \mathcal{L}_{\Sigma \setminus (P_1 \cup P_2)}	= Bel(\kappa) \cap \mathcal{L}_{\Sigma \setminus P_1} \cap \mathcal{L}_{\Sigma \setminus P_2} \\ 
	= &~ Bel(\kappa) \cap Bel(\kappa) \cap \mathcal{L}_{\Sigma \setminus P_1} \cap \mathcal{L}_{\Sigma \setminus P_2} \\
	= &~ (Bel(\kappa) \cap \mathcal{L}_{\Sigma \setminus P_1}) \cap (Bel(\kappa) \cap \mathcal{L}_{\Sigma \setminus P_2}) \\
	= &~ Bel(\kappa_{|\Sigma \setminus P_1}) \cap Bel(\kappa_{|\Sigma \setminus P_2}) \\
	= &~ Bel(\kappa \forgetSm P_1) \cap Bel(\kappa \forgetSm P_2).
\end{align*}
	
Similarly for \dfpessr{5}, we have
\begin{align*}
	& ~ Bel(\kappa \forgetSm P_1 \cup P_2) = Bel(\kappa_{|\Sigma \setminus (P_1 \cup P_2)}) \\
	= &~ Bel(\kappa) \cap \mathcal{L}_{\Sigma \setminus (P_1 \cup P_2)} = Bel(\kappa) \cap \mathcal{L}_{\Sigma \setminus (P_1 \cup (P_1 \cup P_2))} \\
	= &~ Bel(\kappa) \cap (\mathcal{L}_{\Sigma \setminus P_1} \cap \mathcal{L}_{\Sigma \setminus (P_1 \cup P_2)})\\
	= &~ (Bel(\kappa) \cap \mathcal{L}_{\Sigma \setminus P_1}) \cap \mathcal{L}_{\Sigma \setminus (P_1 \cup P_2)} \\
	= &~ Bel(\kappa_{|\Sigma \setminus P_1}) \cap \mathcal{L}_{\Sigma \setminus (P_1 \cup P_2)} = Bel(\kappa_{|\Sigma \setminus P_1}) \cap \mathcal{L}_{(\Sigma \setminus P_1) \setminus P_2} \\
	= &~ Bel((\kappa_{|\Sigma \setminus P_1})_{|(\Sigma \setminus P_1) \setminus P_2}) = Bel((\kappa \forgetSm P_1)_{|(\Sigma \setminus P_1) \setminus P_2}) \\
	= &~ Bel((\kappa \forgetSm P_1) \forgetSm P_2).
\end{align*}

Finally, \Cref{prop:lifting_belief} is used for proving \dfpessr{6}:
\begin{align*}
	& ~ Bel((\kappa \forgetSm P)_{\uparrow \Sigma}) \cap \mathcal{L}_{\Sigma \setminus P} = Bel((\kappa_{|\Sigma \setminus P})_{\uparrow \Sigma}) \cap \mathcal{L}_{\Sigma \setminus P} \\
	= &~ Cn_\Sigma(Bel(\kappa_{|\Sigma \setminus P})) \cap \mathcal{L}_{\Sigma \setminus P} = Bel(\kappa_{|\Sigma \setminus P}) \cap \mathcal{L}_{\Sigma \setminus P} \\
	= &~ Bel(\kappa_{|\Sigma \setminus P}) = Bel(\kappa \forgetSm P). 
\end{align*}
\end{proof}

From \Cref{th:dfpessr_marginalization} above, we can also conclude that the marginalisation forms the signature forgetting
operator that 
only induces minimal changes to the prior beliefs.

\begin{proposition}\label{prop:marginalization_belief}
	Let $\kappa$ be an OCF over signature $\Sigma$, $P \subseteq \Sigma$ a subsignature, and $\forgetS$ an operator satisfying \dfpessr{1}-\dfpessr{6}, where $\kappa \forgetS P$ is an OCF over the reduced signature $\Sigma \setminus P$, then the following relation holds:
	\begin{equation*}
	Bel(\kappa_{|\Sigma \setminus P}) \sat Bel(\kappa \forgetS P)
	\end{equation*}
\end{proposition}

\begin{proof}[Proof of \Cref{prop:marginalization_belief}]
Because of \dfpessr{1} and \dfpessr{6}, we have 
$Bel(\kappa) \sat Bel(\kappa \forgetS P) = Bel((\kappa \forgetS P)_{\uparrow \Sigma}) \cap \mathcal{L}_{\Sigma \setminus P}$, hence 
 $\kappa \sat \phi$ for all $ \phi \in Bel((\kappa \forgetS P)_{\uparrow \Sigma}) \cap \mathcal{L}_{\Sigma \setminus P}$. 
 But then also, due to \Cref{lem:marginalization_equiv}, 
 $\kappa_{|\Sigma \setminus P} \sat \phi$ for all $\phi \in Bel((\kappa \forgetS P)_{\uparrow \Sigma}) \cap \mathcal{L}_{\Sigma \setminus P}$, which means $Bel(\kappa_{|\Sigma \setminus P}) \sat Bel((\kappa \forgetS P)_{\uparrow \Sigma}) \cap \mathcal{L}_{\Sigma \setminus P}$, and therefore again due to  \dfpessr{6}, $Bel(\kappa_{|\Sigma \setminus P}) \sat Bel(\kappa \forgetS P)$, which was to be shown. 
\end{proof}

Thus, we know that any signature forgetting operator other than the marginalisation must induce
further belief changes for some epistemic states and signatures. Such signature forgetting
operators could for example depend on some model prioritisation in addition to the epistemic
state and the signature itself.

\section{Forgetting Signatures vs. Forgetting Formulas -- A Triviality Result}\label{sec:triviality_result}

In the following, we want to discuss \dfp{1}-\dfp{7} displaying the right properties to describe
the general notions forgetting. In our opinion, these properties might display the \emph{right} properties when assuming forgetting as a reduction of the language, or as forgetting signature elements, respectively. Delgrande also comments on this, and argues that other (belief change) operators that could be considered as some kind of forgetting, e.g. contraction, should simply not be
considered as forgetting. We think that this view on the concept of forgetting as such is
debatable, since there exist multiple intuitively and cognitively different kinds of forgetting
\cite{beierle2019towards} from which Delgrande’s approach, which corresponds to the notion of focussing (\Cref{th:marginalization_delgrande_equiv}), only forms one. Therefore, it is still to be investigated whether \dfp{1}-\dfp{7} also states the right properties for other kinds of forgetting.

Following the overview of cognitively different kinds of forgetting presented in \cite{beierle2019towards}, it can be seen that the concept of focussing, i.e. the marginalisation, forms the only kind of forgetting that describes forgetting with respect to signatures. Thus, in order to investigate Delgrande’s forgetting properties for those kinds of forgetting that argue about formulas instead, we have to generalise and extend \dfp{1}-\dfp{7} such that they not only argue about arbitrary epistemic states and operators, but also about formulas. We refer to them as \dfpesl{1}-\dfpesl{6}. For this, let $\Psi, \Phi$ be epistemic states, $\phi, \psi \in \mathcal{L}$ formulas, and $\forgetL$ an arbitrary belief change operator:

\begin{description}
	\item[\dfpesl{1}] $Bel(\Psi) \sat Bel(\Psi \forgetL \phi)$
	\item[\dfpesl{2}] If $Bel(\Psi) \sat Bel(\Phi)$, then $Bel(\Psi \forgetL \phi) \sat Bel(\Phi \forgetL \phi)$
	\item[\dfpesl{3}] If $\phi \sat \psi$, then $Bel(\Psi \forgetL \phi) \equiv Bel((\Psi \forgetL \psi) \forgetL \phi)$
	\item[\dfpesl{4}] $Bel(\Psi \forgetL (\phi \vee \psi)) \equiv Bel(\Psi \forgetL \phi) \cap Bel(\Psi \forgetL \psi)$
	\item[\dfpesl{5}] $Bel(\Psi \forgetL (\phi \vee \psi)) \equiv Bel((\Psi \forgetL \phi) \forgetL \psi)$
	\item[\dfpesl{6}] If $\phi \not\equiv \top$, then $Bel(\Psi \forgetL \phi) \nsat \phi$
\end{description}

While the extension to \dfpesl{1}-\dfpesl{6} works almost analogously to the extension to
\dfpessr{1}-\dfpessr{6}, there exist some crucial differences, which we will address in the
following. In \dfp{4}, Delgrande argues about forgetting signature $P, P’$ for which we assume
that $P’$ is fully included in $P$. In order to extend and generalise this property, we have to
examine how this notion can be described with respect to formulas. We found it most accurate to
generalise this relation of the information we would like to forget by means of the specificity of
formulas, i.e. $\phi \sat \psi$. Thereby, we say that the knowledge described by $\psi$ is fully
included in that of $\phi$, if and only if $\psi$ can be inferred from $\phi$. More formally, this can be stated by means of the deductive closures of $\phi$ and $\psi$, i.e. $\phi \sat \psi \Leftrightarrow Cn(\psi) \subseteq Cn(\phi)$.

In \dfp{5} and \dfp{6}, Delgrande argues about forgetting two signatures $P, P’$ at once, which is
described as forgetting $P \cup P’$. On a more intuitive level this can be viewed as only forgetting
a single piece of information that consist of both the information we actual like to forget. When
arguing about formulas instead of signatures, this can be expressed by means of a disjunction $
\phi \vee \psi$, where $\phi$ and $\psi$ are the two formulas we actually want to forget. Even
though it might seem more appropriate to describe this idea by means of a conjunction $\phi
\wedge \psi$, it is not sufficient to forget the conjunction in order to forget both $\phi$ and $\psi$, since it is generally sufficient to forget one of the formulas in order to forget the conjunction as well. Thus, describing the unification of two pieces of information by means of a disjunction guarantees that both formulas can no longer be inferred after forgetting.

Just as for the postulates for forgetting signatures, we omit \dfp{3}, since a belief set is already
deductively closed by definition. Furthermore, we omit \dfp{7} since it argues about the relation of
forgetting in the reduced and in the original language, which is not applicable in case of forgetting
formulas. Instead, we introduce an additional postulate \dfpesl{6} that explicitly states the success of the forgetting operator, i.e. after forgetting a non-tautologous formula $\phi$, we are no longer able to infer $\phi$.

When extending \dfp{1}-\dfp{7} to forgetting formulas, Delgrande's idea that forgetting should be performed on the knowledge level, and therefore should be independent of the syntactic structure of the given knowledge also extends to the knowledge we want to forget. Thus, we show in \Cref{th:syntax_independence} the syntax independence implied by \dfpesl{1}-\dfpesl{6}. 

\begin{theorem}[Syntax Independence]\label{th:syntax_independence}
	Let $\Psi$ be an epistemic state and $\forgetL$ a belief change operator satisfying \dfpesl{1}-\dfpesl{6}. Further, let $\phi, \psi \in \mathcal{L}$ be formulas, then the following holds:
	\begin{equation*}
		\text{If } \phi \equiv \psi, \tthen Bel(\Psi \forgetL \phi) \equiv Bel(\Psi \forgetL \psi).
	\end{equation*}
\end{theorem}

\begin{proof}[Proof of \Cref{th:syntax_independence}]
From $\phi \equiv \psi$, we obtain with \dfpesl{3} and  \dfpesl{5}:
\begin{align*}
	&~ Bel(\Psi \forgetL \phi) = Bel((\Psi \forgetL \psi) \forgetL \phi) =  Bel(\Psi \forgetL \psi \vee \phi)
\end{align*} and
\begin{align*}
	 &~ Bel(\Psi \forgetL \psi) = Bel((\Psi \forgetL \phi) \forgetL \psi) = Bel(\Psi \forgetL \phi \vee \psi).
\end{align*}
Therefore,
\begin{align*}
	Bel(\Psi \forgetL \phi) &=  Bel(\Psi \forgetL \psi) \cap Bel(\Psi \forgetL \phi)\\ &= Bel(\Psi \forgetL \psi),
\end{align*}
due to \dfpesl{4}. 
\end{proof}

Next, we show that there cannot exist any non-trivial belief change operator satisfying \dfpesl{1}-\dfpesl{6}. For this, we first show that \dfpesl{3} together with \dfpesl{5} imply that the forgetting of any conjunction $\phi \wedge \psi$ must result in beliefs equivalent to just forgetting $\phi$ or $\psi$ (\Cref{prop:forgetting_conjunction_components}).

\begin{proposition}\label{prop:forgetting_conjunction_components}
	Let $\Psi$ be an epistemic state and $\forgetL$ a belief change operator satisfying \dfpesl{1}-\dfpesl{6}, then
	\begin{equation*}
	Bel(\Psi \forgetL \phi) \equiv Bel(\Psi \forgetL \phi \wedge \psi) \equiv Bel(\Psi \forgetL \psi)
	\end{equation*}
	holds for all formulas $\phi, \psi \in \mathcal{L}$.
\end{proposition}

\begin{proof}[Proof of \textnormal{(\Cref{prop:forgetting_conjunction_components})}]
Using \dfpesl{3}, \dfpesl{5}, and \Cref{th:syntax_independence}, we compute
\begin{align*}
	& ~Bel(\Psi \forgetL \phi \wedge \psi) = Bel((\Psi \forgetL \phi) \forgetL \phi \wedge \psi) \\
	= &~ Bel(\Psi \forgetL \phi \vee (\phi \wedge \psi)) = Bel(\Psi \forgetL \phi).
\end{align*}
This  holds for $\psi$ analogously. Thus, we can conclude $Bel(\Psi \forgetL \phi) \equiv Bel(\Psi \forgetL \phi \wedge \psi) \equiv Bel(\Psi \forgetL \psi)$.
\end{proof}

From \Cref{prop:forgetting_conjunction_components} we can especially conclude that forgetting according to \dfpesl{1}-\dfpesl{6} must be independent of the formula we actually like to forget (\Cref{cor:forgetting_always_equivalent}).

\begin{corollary}\label{cor:forgetting_always_equivalent}
	Let $\Psi$ be an epistemic state and $\forgetL$ a belief change operator satisfying \dfpesl{1}-\dfpesl{6}, then
	\begin{equation*}
	Bel(\Psi \forgetL \phi) \equiv Bel(\Psi \forgetL \psi)
	\end{equation*}
	holds for all formulas $\phi, \psi \in \mathcal{L}$.
\end{corollary}

Therefore, we know that a forgetting operator satisfying \dfpesl{1}-\dfpesl{6} must always forget all prior beliefs except for tautologies (\Cref{th:forget_tautology}).

\begin{theorem}[Triviality Result]\label{th:forget_tautology}
	Let $\Psi$ be an epistemic state. A belief change operator $\forgetL$ satisfies \dfpesl{1}-\dfpesl{6}, if and only if $Bel(\Psi \forgetL \phi) \equiv \top$ holds for each $\phi \in \mathcal{L}$.
\end{theorem}

\begin{proof}[Proof of \Cref{th:forget_tautology}]
	We prove \Cref{th:forget_tautology} in two steps. First, we show that if a belief change operator satisfies \dfpesl{1}-\dfpesl{6}, then it must always result in posterior beliefs $Bel(\Psi \forgetL \phi)$ equivalent to $\top$. Second, we show that each belief change operator $\forgetL$ with $Bel(\Psi \forgetL \phi) \equiv \top$ for each $\phi \in \mathcal{L}$ satisfies \dfpesl{1}-\dfpesl{6}. We refer to these two steps as $(\Rightarrow)$ and $(\Leftarrow)$. Note that we assume all formulas $\phi, \psi \in \mathcal{L}$ to be non-tautologous.\\
	
	\textit{Case} $(\Rightarrow)$:
	From \Cref{cor:forgetting_always_equivalent}, we know that applying $\forgetL$ to an epistemic state $\Psi$ must result in equivalent beliefs for all formulas $\phi, \psi \in \mathcal{L}$.
	From \dfpesl{6} we know that after forgetting a formula $\phi$, we are no longer able to infer $\phi$. Since the posterior beliefs are equivalent for all formulas, we can conclude that after applying $\forgetL$ to $\Psi$, we are not able to infer any formula, but tautologies.
	\begin{align*}
		& \begin{aligned}
			& Bel(\Psi \forgetL \phi) \equiv Bel(\Psi \forgetL \psi), \\
			& \quad \tall \phi, \psi \in \mathcal{L}
		\end{aligned} & \textnormal{(\Cref{cor:forgetting_always_equivalent})} \\
		\Rightarrow & Bel(\Psi \forgetL \phi) \nsat \phi, \psi, \tall \phi, \psi \in \mathcal{L} & \textnormal{\dfpesl{1}} \\
		\Leftrightarrow & ~ Bel(\Psi \forgetL \phi) \equiv \top, \tall \phi \in \mathcal{L}
	\end{align*}
	
	\textit{Case} $(\Leftarrow)$:
	Let $\Psi$ and $\Phi$ be epistemic states and $\phi, \psi \in \mathcal{L}$ be non-tautologous formulas, and $\forgetL$ a belief change operator with $Bel(\Psi \forgetL \phi) \equiv \top$ for all epistemic states $\Psi$ and formulas $\phi$. The fact that $\forgetL$ satisfies \dfpesl{1}-\dfpesl{6} directly concludes from the assumption $Bel(\Psi \forgetL \phi) \equiv \top$, for all $\phi \in \mathcal{L}$.

	We showed that both cases $(\Rightarrow)$ and $(\Leftarrow)$ hold, and therefore proved the triviality result stated in \Cref{th:forget_tautology}.
\end{proof}

\section{Conclusion}\label{sec:conclusion}
We discussed two of the existing approaches towards a general forgetting framework. The first approach was that of Delgrande \shortcite{delgrande2017knowledge} in which he gives a general forgetting definition that argues about forgetting on the knowledge level, and is capable of representing several of the hitherto existing logic-specific forgetting approaches, such as Boole's atom forgetting in propositional logic \cite{boole1854investigation}. The second approach was that of Beierle et al. \shortcite{beierle2019towards}. In contrast to Delgrande's approach, Beierle et al. define several cognitively different kinds of forgetting in a general OCF framework, which is generally more expressive than just arguing about knowledge sets. Thereby, we concretely focussed on the marginalisation or the concept of focussing as one kind of forgetting, respectively, which is of importance when it comes to efficient and focussed query answering. We showed that the marginalisation generalizes Delgrande's forgetting definition to epistemic states by resulting in equivalent posterior beliefs, as well as holding the same properties, which Delgrande referred to as \emph{right} and \emph{desirable}. Furthermore, this implies that the relations Delgrande elaborated between his and the logic-specific approaches also hold for the marginalisation. We exemplarily showed this by means of Boole's atom forgetting in propositional logic.
We think that \dfp{1}-\dfp{7}, or \dfpessr{1}-\dfpessr{6} respectively, describe properties that are \emph{right} and \emph{desirable} as long as we consider the forgetting of signature elements. 
However, we showed that these properties are not suitable for postulating properties for any kind of forgetting formulas, since generalizing these properties to formulas \dfpesl{1}-\dfpesl{6} implies the triviality result stated in \Cref{th:forget_tautology}.

In principle, we agree with Delgrande insofar that belief change operators like contraction are essentially different from the notion of forgetting as it is implemented by Delgrande's approach. However, we argue that Delgrande's approach and in general, approaches based on variable elimination, are too narrow to cover cognitive forgetting in its full generality. As our triviality result shows, Delgrande's postulates seem to be unsuitable for describing  the forgetting of formulas. Nevertheless, as the works of Beierle et al. \shortcite{beierle2019towards} show, very different kinds of forgetting are realizable in a common framework, distinguishable by different properties. So, as part of our future work, we pursue the research question which of Delgrande's postulates (which all seem very rational at first sight) need to be modified or omitted to make the idea of forgetting by variable elimination reconcilable to other forms of forgetting and how Delgrande's forgetting definition itself could be amended to satisfy the adapted postulates.

\bibliographystyle{kr}
\bibliography{kr-sample}

\end{document}